\documentclass[sn-mathphys-ay]{sn-jnl}

\usepackage{graphicx}%
\usepackage{multirow}%
\usepackage{amsmath,amssymb,amsfonts}%
\usepackage{amsthm}%
\usepackage{mathrsfs}%
\usepackage[title]{appendix}%
\usepackage{xcolor}%
\usepackage{textcomp}%
\usepackage{manyfoot}%
\usepackage{booktabs}%
\usepackage{algorithm}%
\usepackage{algorithmicx}%
\usepackage{algpseudocode}%
\usepackage{listings}%
\usepackage{enumerate}

\theoremstyle{thmstyleone}%
\newtheorem{theorem}{Theorem}
\theoremstyle{thmstyletwo}%
\newtheorem{example}{Example}%

\theoremstyle{thmstylethree}%
\newtheorem{definition}{Definition}%

\begin{document}

\title[Modeling Local Search Metaheuristics Using Markov Decision Processes]{Modeling Local Search Metaheuristics Using Markov Decision Processes}


\author{\fnm{Rub\'en} \sur{Ruiz-Torrubiano}}\email{ruben.ruiz@imc.ac.at}

\affil{\orgname{IMC Krems University of Applied Sciences}, \orgaddress{\street{Piaristengasse 1}, \city{Krems}, \postcode{3500}, \country{Austria}}}


\abstract{
    Local search metaheuristics like tabu search or simulated annealing are popular heuristic optimization algorithms for finding near-optimal 
    solutions for combinatorial optimization problems. However, it is still challenging for researchers and 
    practitioners to analyze their behaviour and systematically choose one over a vast set of possible metaheuristics for the particular problem
    at hand. In this paper, we introduce a theoretical framework based on Markov Decision Processes (MDP) for 
    analyzing local search metaheuristics. This framework not only helps in providing convergence 
    results for individual algorithms, but also provides an explicit characterization of 
    the exploration-exploitation tradeoff and a theory-grounded guidance for practitioners for choosing 
    an appropriate metaheuristic for the problem at hand. We present this framework in detail and 
    show how to apply it in the case of hill climbing
    and the simulated annealing algorithm.    
}




\maketitle

\section{Introduction}\label{sec1}

Metaheuristics are high level algorithmic frameworks that provide guidance or strategies for developing 
heuristic optimization algorithms \citep{sorensen_metaheuristics_2013}. Frequently, these strategies use concepts 
from other disciplines as their underlying inspiration, like the evolution process in biology, the process of 
cooling a metallic alloy, the cooperative behaviour of flocks of birds, and many more \citep{silberholz_comparison_2010}.
Despite being heuristic in nature, so that in general there is no general theoretical guarantee that the algorithm
will find the global optimum, metaheurstics have demonstrated a tremendous success in practice, being 
effectively used in applications like logistics (e.g. vehicle routing problem \citep{lin_applying_2009}), 
network design \citep{fernandez_metaheuristics_2018}, scheduling \citep{pillay_survey_2014, kaur_review_2021, EcCS22, CeGS20},
finance (e.g. portfolio optimization \citep{doering_metaheuristics_2019}) and many others. 

In general, it can be distiguished between 
local search (or single-solution) metaheuristics like simulated annealing (SA) \citep{kirkpatrick_optimization_1983}, 
tabu search (TS) \citep{glover_tabu_1998}, variable neighborhood search (VSN) \citep{mladenovic_variable_1997} and 
large neighborhood search (LNS) \citep{shaw_using_1998}, on 
the one side, and population-based metaheuristics on the other, like genetic algorithms (GA) \citep{holland_genetic_1992}
particle swarm optimization (PSO) \citep{kennedy_particle_1995} or estimation of distribution algorithms (EDA)
\citep{larranaga_estimation_2001}. Local search metaheuristics typically start with a candidate solution $x_0$ 
(not necesarilly feasible) and apply stochastic operators to select the next candidate solution from a suitable chosen neighborhood 
or pool of individuals. By iteratively repeating this process, local search metaheuristics are able to find 
near-optimal solutions to the problem at hand efficiently, even if the search space is very large. By contrast, 
population-based metaheuristics start with a group of candidate solutions (often referred to as population) 
$P=\{x_1, x_2, \dots, x_n \}$ and apply different types of operators to generate new candidate solutions, updating 
the population in the process.

It is well acknowledged that there is a lack of strong theoretical frameworks for understanding and analyzing the 
behaviour of metaheuristic algorithms \citep{liu_unified_2011, chauhdry_framework_2023}. While theoretical results and convergence analysis do 
exist for individual algorithms like genetic algorithms and simulated annealing \citep{suzuki1995markov, suzuki1998further}, 
existing frameworks are either focused on one type of analysis (like convergence analysis \citep{liu_unified_2011}), 
or they provide only weak results \citep{chauhdry_framework_2023}. Moreover, they provide no theory-grounded 
guidance for practitioners on which algorithm might be most appropriate a particular task.

In this paper, we try to fill this gap by proposing a novel framework based on Markov Decision Processes (MDP) 
as the underlying model for metaheuristic algorithms. Specifically, in this study we focus on local search 
metaheuristics, leaving population-based metaheuristics for future work. We show that particular instances 
of local search metaheuristics can be modeled as policies of a suitably defined MDP and, based on that, we 
derive different coefficients to model their convergence and exploration-exploitation behaviour. By introducing different parameters into these 
bounds and assigning them a practical meaning, we show which properties have to be taken into account for a given 
metaheuristic (policy) to perform well on a particular class of problems. 

This paper is structured as follows. In Section \ref{sec2}, we review previous relevant work. Section \ref{sec3} presents 
the MDP framework for local search metaheuristics and applies it to the hill climbing algorithm as a first example. 
In Section \ref{sec4} we apply this framework to simulated annealing, showing how the relevant measures of convergence 
and exploration-exploitation can be explicitly calculated.
Section \ref{sec5} concludes the paper and outlines future work.

\section{Previous Work}
\label{sec2}

Mathematical analysis of optimization metaheuristics has been done previously in the literature in a variety 
of ways. The work on the schema theorem by Holland \citep{holland_adaptation_1975} can be considered as one of the first 
attempts at mathematical formalization of metaheuristic algorithms. This result applied only to simple GAs with 
one-point crossover and binary encoding. One of the most popular variants of mathematical analysis
 was sparheaded by the Markov chain analysis of genetic algorithms (GA) which was mostly done in the 1990s. For instance, 
in early work the authors of \citep{eiben_global_1991} proved that GAs with an elitist selection scheme converge to the optimum 
with probability 1. Suzuki \citep{suzuki1995markov, suzuki1998further} proved theoretical results based on Markov chain analysis 
for GAs focusing on mutation probabilities and developed an early comparison with simulated annealing (SA). In \citep{cao_convergence_1997},
the authors applied a non-stationary Markov model for the convergence analysis of genetic algorithms. One of the main 
drawbacks of using Markov chain theory for metaheuristic analysis is that Markov chains are not expressive enough for handling 
important aspects of the search process in metaheuristics. Critically, the reward structure and the different operators used in the algorithm  
can only be expressed implicitly via the transition matrix. 

More recent approaches in the litereature use other models, like the evolutionary Turing machine \citep{eberbach_toward_2005} and 
the nested partition algorithm \citep{chauhdry_framework_2023}. However, most of these results focus on convergence analysis and 
do not handle the exploration-exploitation dilemma explicitly. This question was addressed in \citep{cuevas_experimental_2021} by 
experimental means, but no theoretical analysis was provided. 
Similarly, a case study based on elitist evolutionary algorihtms (concretely, random univariate search 
and evolutionary programming) was presented in \citep{chen_exploitation_2020} using a
specially-tailored probabilistic framework. An extensive survey on the exploration-exploitation tradeoff in metaheuristics 
is given in \citep{xu_exploration-exploitation_2014}. In \citep{chen_optimal_2009}, a mathematical characterization of 
the exploration-exploitation tradeoff is given in the form of the optimal contraction theorem. However, their definition 
of exploration leaves out the possibility that information collected during the search can be used for exploring other areas 
of the search space.

In contrast to previous approaches, our framework applies Markov decision processes (MDP) to the problem of modeling the 
convergence behaviour and the exploration-exploitation tradeoff of local search metaheuristics. We consider MDP to be a more  
suitable framework than the previously presented approaches since the reward structure of the general search problem can 
be represented explicitly. Moreover, transformations and sequences of candidate solutions can be modeled explicitly by 
means of actions, which in MDP are naturally linked to rewards. We therefore see a particular metaheuristic algorithm as a 
stochastic agent that makes decisions about which candidate solution to visit next based on the information obtained so far and 
the reward obtained. We explore the implications of this modeling choice and we show in the following sections that  
the convergence and the exploration-exploitation analysis of metaheuristic algorithms can be calculated explicitly.

\section{Markov Decision Processes for Local Search Metaheuristics}
\label{sec3}

In this section we develop the mathematical machinery needed to define a MDP model for local search metaheuristics. Our main insight is that we can model 
any local search metaheuristic as a stochastic agent that executes some policy on a suitably defined MDP. 
Critically, we note that even if the goal of a specific metaheuristic is to find the optimum of an objective 
function, the algorithm itself proceeds by maximizing the total reward over a (possibly) infinite time horizon. Clearly, 
the rewards are directly connected with the best values found so far of the objective function, but the 
stochastic agent needs to proceed by gathering information about promising directions in the search space (exploration-exploitation dilemma).
Therefore, the exploration of the search space is guided mainly by the maximization of the total reward.  

\subsection{Definitions}

We start by introducing some preliminaries 
in the form of notation and definitions. Then we specify the model in detail and describe how to define 
metaheuristic policies using this framework. In the following, we will mainly follow the conventions and definitions
used in \citep{kallenberg_lecture_2022}.

\begin{definition}[Markov Decision Process]
    A Markov Decision Process (MDP) is a tuple $\langle \mathcal{S}, \mathcal{A}, \mathcal{P}, \mathcal{R}, \alpha \rangle$, where:
    \begin{itemize}
        \item $\mathcal{S}$ is a set of states.
        \item $\mathcal{A}(i)$ is a set of actions available in state $i \in \mathcal{S}$.
        \item $\mathcal{P}(a)=\{p_{ij}(a)\}_{ij}$ is the state transition matrix, given action $a$,
            where $p_{ij}(a)=\mathbb{P}\{ S_{t+1}=j | S_t=i, A_t=a\}$.
        \item $\mathcal{R}(a)$ is a reward function, $\mathcal{R}_s(a)=\mathbb{E}[r_{t+1}|S_i=s, A_t=a]$.
        \item $\alpha\in[0,1]$ is a discount factor.
    \end{itemize}
\end{definition}

In words, a MDP is a stochastic decision model where an agent makes decisions on each time step about
which action $a$ to take from a predefined action set $\mathcal{A}(i)$, which might be 
different for each state $i\in\mathcal{A}(i)$. Most importantly, the transition probabilities
to a new state $S_{t+1}$ depend only on the last state $S_t$, and not on the previous 
history (i.e. the previously visited states). This is known as the \emph{Markov property}.

\begin{definition}
    A \emph{policy} $R$ is a decision rule $R=(\pi^1, \pi^2, \dots, \pi^t, \dots )$, that 
    can be expressed as a probability distribution of the actions, given states: 
    \[\pi_{ia}^t=\mathbb{P}\{A_t=a|S_t=s\}\]
    If the policy $R$ is such that there exist an action $a$ such that $\pi_{ia}^t=1$ for every 
    state $i\in \mathcal{S}$, we say the the policy is \emph{deterministic}, otherwise,
    we call the policy \emph{stochastic}. If the policy does not depend on time, we call this policy 
    \emph{stationary} and simply write $\pi_{ia}$.
\end{definition}

The natural problem that now arises is how find policies that maximize the rewards obtained
in some precise sense. This is captured by means of an utility function, which might be e.g. 
maximize the total expected or average reward. A MDP can be classified according to its time horizon in finite or infinite time MDP. Likewise, 
we can also distinguish between discrete and continuous time MDPs. For local search 
metaheuristics, we focus on the MDP formulation where:

\begin{enumerate}[(a)]
    \item The time is discrete and the time horizon is infinite. \label{conditionA}
    \item The utility function is the total expected reward. \label{conditionB}
\end{enumerate}

Regarding condition (\ref{conditionA}), we note that clearly the execution time of a metaheuristic 
is always finite, but this finite time cannot be fixed beforehand (except for the class 
of metaheuristics that are run for a fixed number of iterations). This is modeled by having 
terminal states with no outgoing transitions (i.e. the probability of transitioning 
to any other state is 0, and $\mathcal{A}(i)=\emptyset$ for all terminal states $i\in\mathcal{S}$).
We argue for condition (\ref{conditionB}) that, in general, the main strategy used by 
all metaheuristics is to search for the optimum \emph{indirectly} by exploring the search 
space $\mathcal{X}$ and exploiting promising regions where an optimum $x^*\in\mathcal{X}$ is 
supposed to be. However, this is done by \emph{sampling} from the search space in some 
specific sense and evaluating the objective function $f$ on that sample, thus collecting 
information about the target function value. If we interpret this sampling and evaluating 
procedure as collecting \emph{rewards}, the agent (i.e. the metaheuristic) can be validly
interpreted as trying to maximize the total reward obtained. 

\subsection{Optimal policies}

In general, we therefore model local search metaheuristic algorithms by means of 
policies $R$. We define the reward of such a policy as the total (discounted) expected reward:

\begin{equation}
v_i^\alpha(R)=\sum_{t=1}^\infty \mathbb{E}_{i,R} \{ \alpha^{t-1} r_t \}
\end{equation}

where $r_t$ denotes the collected actual reward at time step $t$, and the expectation 
is over policy $R$, given initial state $i\in\mathcal{S}$, and $\alpha$ is a discount 
factor that regulates the value of future rewards. In the following, we will asume $\alpha=1$ so we can write:

\begin{equation}
    v_i(R)=\sum_{t=1}^\infty \sum_{j,a} \mathbb{P}\{ S_{t+1} = j | S_t=i, A_t = a\} r_{j}(a)
\end{equation}

Let $\{P(\pi^t)\}_{ij}$ denote the transition matrix of policy $R$, and $\{r(\pi^r)\}_{i}$
its reward vector. We have

\begin{align}
\{P(\pi^t)\}_{ij}&=\sum_{a\in\mathcal{A}} p_{ij}^t(a) \pi_{ia}^t \text{ for all } i, j\in\mathcal{S} \\
\{r(\pi^r)\}_{i}&=\sum_{a\in\mathcal{A}} r_i^t(a) \pi_{ia}^t \text{ for all } i \in\mathcal{S}
\end{align}

The reward of policy $R$ can be now expressed based on the transition probabilities as follows:

\begin{equation}
v^\alpha(R)=\sum_{t=1}^\infty \alpha^{t-1} P(\pi^1)P(\pi^2)\dots P(\pi^{t-1})r(\pi^t)
\end{equation}

\begin{definition}
We call a policy $R_*$ \emph{optimal} if it is the best achievable policy considering the reward function, i.e.

\[
v^\alpha(R_*)=\operatorname{sup}_R v^\alpha(R)
\]

\end{definition}

\subsection{Local search metaheuristics as policies}

We now formally represent an arbitrary local search metaheuristic as a policy in a MDP as  
follows. In this study, we focus on binary strings of length $n$ as candidate solutions, therefore $\mathcal{X}\subseteq\{0,1\}^n$.
From now on, we denote by $\mathbf{x}\in\{0,1\}^n$ a binary vector which is a candidate solution for the problem at hand, which
might be feasible or not (i.e. it can be that, during the search process, there is an $\mathbf{x}_t$ such 
that $\mathbf{x}_t\notin\mathcal{X}$). In the following, we will assume a maximization problem, i.e. find $\mathbf{x}^*$ such that
$\mathbf{x}^* = \operatorname{argmax}_{\mathbf{x}\in\mathcal{X}} f(\mathbf{x})$.

\begin{definition}
We define a \emph{neighborhood} $\mathcal{N}_H(\mathbf{x})$ as the set of candidate solutions $\mathbf{x}'\in\mathcal{X}$ which are neighbors 
of $\mathbf{x}$ according to a criterium $H$, i.e. $\mathcal{N}_H(\mathbf{x})=\{ x'\in\mathcal{X} \, \vert \, H(\mathbf{x}';\mathbf{x})=1\}$.
\end{definition}

Let $\mathbf{x}\in\mathcal{X}$ and $\mathbf{x}'\in\mathcal{N}_H(\mathbf{x})$. We denote by $\mathbf{x}\mapsto\mathbf{x}'$  
the operator that transforms $\mathbf{x}$ into $\mathbf{x}'$. For example, the neighborhood of all binary strings located at Hamming distance one from $\mathbf{x}$, would 
be defined by $\mathcal{N}_{H_D}(\mathbf{x})=\{ x'\in\mathcal{X} \, \vert \, H_D(\mathbf{x}';\mathbf{x})=1\}$. In this case, 
$\mathbf{x}\mapsto\mathbf{x}'$ denotes the operator that flips the one bit necessary to turn $\mathbf{x}$ into $\mathbf{x}'$.

\begin{definition}[Local search MDP]
Let our local search MDP be defined by:

\begin{itemize}
    \item $\mathcal{S}=\{0,1,\dots,2^{n-1}\}$. Any state $i\in\mathcal{S}$ can trivially be converted into a candidate solution as the 
        binary representation of the integer $i$.
    \item $\mathcal{A}(i)=\{ i\mapsto j \, \vert \, j\in \bigcup_k \mathcal{N}_k\}$, the set of all possible transformations in all neighborhoods applicable to the current state $i\in\mathcal{S}$.
    \item The transition probabilities, for all $i,j\in\mathcal{S}$, $a\in\mathcal{A}(i)$:
    \begin{equation}
        p_{ij}(a)=\begin{cases}
        1 / |\mathcal{A}(i)| \text{, if } a = i\mapsto j \\
        0 \text{, otherwise}
        \end{cases}
    \end{equation}
    \item The reward for state $j$ given action $a=i\mapsto j$ is defined by $r_{i}(a)=f(j)-f(i)$, where $f$ is the 
    objective function\footnote{Note the slight abuse of notation here since $f$ is defined on $\mathcal{X}$ rather than $\mathcal{S}$, 
    but the conversion is trivial.}. 
\end{itemize}
\end{definition}

An instantiation of this MDP therefore starts at an arbitrary state $i\in\mathcal{S}$ and visits a number of states 
according to the transition probabilities before halting at time $T$. The algorithm terminates as soon as it reaches a
\emph{terminal} or \emph{absorving} state, i.e. a state $i\in\mathcal{S}$ where the transition probability of reaching any other state is 0
($p_{ij}(a)=0$ for all $i\neq j$, and $p_{ii}(a)=1$ otherwise). Note that the rewards are chosen in such a way that, in general, 
they incentivize making progress towards the optimum $\mathbf{x}^*$. 

In this setting, we identify a policy $R$ with a particular local search metaheuristic. 

\begin{example}[Hill climbing]  
    We define a policy $R_{HC}$ for the standard hill climbing algorithm as follows:
    \begin{itemize}
        \item $\mathcal{A}(i)=\{ i\mapsto j \, \vert \, j\in \mathcal{N}_{H_D}\}$ as defined before (we flip one bit), and 
        \item Let $\mathcal{M}(i)=\{j\in\mathcal{S} \, \vert \, j = \operatorname{argmax} f(j'), a=i\mapsto j, a\in\mathcal{A}(i)\}$, then
        \begin{equation}
            \pi_{ia}=\begin{cases}
            1/|\mathcal{M}(i)| \text{, if } j \in \mathcal{M}(i)\\
            0 \text{, otherwise}
            \end{cases}
        \end{equation}
    \end{itemize}
    Note that $R_{HC}$ is a deterministic policy except when $|\mathcal{M}(i)|>1$, in which case ties are broken randomly.
\end{example}

In general, the MDP is defined in such a way that the balance between exploration and exploitation becomes explicit in the 
definition of the policy. Informally, the agent can either maximize the reward in the short run (thus going into an exploitation phase),
or temporarily accept negative rewards and therefore going into a exploration phase. In the following definitions, we 
define exploration and exploitation explicitly using these notions.

\begin{definition}[Exploration and exploitation]
    Let $a\in\mathcal{A}(i)$ be an action such that $a=i\mapsto j$. We say that $a$ is an \emph{exploration} action
    if and only if $r_j(a) \leq 0$. Otherwise, we call $a$ and \emph{exploitation} action.
\end{definition}

\begin{definition}[Exploration-exploitation function]
    Let $\sigma:\mathcal{A}(i)\rightarrow\{0,1\}$ be a function such that $\sigma(a)=1$ if and only if $a$ is an exploration 
    action, and $\sigma(a)=0$ otherwise. We call $\sigma$ the \emph{exploration-exploitation} function.
\end{definition}

We now state our main result which 
shows that any local search metaheuristic can be represented by a policy where the balance between exploration and 
exploitation becomes explicit.

\begin{theorem}[Local search exploration-exploitation theorem]
    \label{theorem1}
    Let $M$ be a local search MDP. For any local search metaheuristic $A$, there exist a policy $R_A$ such that
    \begin{equation}
        \pi_{ia}^t=\alpha_i^A(t) \mathbb{P}\{\sigma(a)=1 \, \vert \, S_t=i\} + \beta_i^A(t) \mathbb{P}\{\sigma(a)=0 \, \vert \, S_t=i\}
    \end{equation}
\end{theorem}

\begin{proof}
    Let $\mathcal{M}_t^+(i)=\{j\in\mathcal{S} \, \vert \, f(j)>f(i), a=i\mapsto j, a\in\mathcal{A}(i), S_t=i\}$ and
    $\mathcal{M}_t^-(i)=\{j\in\mathcal{S} \, \vert \, f(j)\leq f(i), a=i\mapsto j, a\in\mathcal{A}(i), S_t=i\}$ be the
    sets of improving and non-improving states given action $a$, respectively. We have:
    \begin{align}
        \pi_{ia}^t&=\mathbb{P}\{ A_t = a \, \vert \, S_t = i \} \nonumber \\
            &= \mathbb{P}\{ A_t = a, j\in\mathcal{M}_t^+(i)\, \vert \, S_t = i\} + 
              \mathbb{P}\{ A_t = a, j\in\mathcal{M}_t^-(i)\, \vert \, S_t = i\} \nonumber \\
            &= \frac{|\mathcal{M}_t(i)^+|}{|\mathcal{M}_t(i)^+|+|\mathcal{M}_t(i)^-|}\mathbb{P}\{\sigma(a)=0 \, \vert \, S_t=i\} + \nonumber \\
            &+\frac{|\mathcal{M}_t(i)^-|}{|\mathcal{M}_t(i)^+|+|\mathcal{M}_t(i)^-|}\mathbb{P}\{\sigma(a)=1 \, \vert \, S_t=i\}
    \end{align}
    Therefore,
    \[
    \alpha_i^A(t)=\frac{|\mathcal{M}_t(i)^-|}{|\mathcal{M}_t(i)^+|+|\mathcal{M}_t(i)^-|} \text{, and }\beta_i^A(t)=\frac{|\mathcal{M}_t(i)^+|}{|\mathcal{M}_t(i)^+|+|\mathcal{M}_t(i)^-|}
    \]
\end{proof}

Theorem \ref{theorem1} provides hints to several interesting measures. The first one is a measure that can be 
used for convergence analysis and can be defined as follows: 

\begin{equation}
    \gamma^A_i(t):=\cfrac{\beta_i^A(t)}{\alpha_i^A(t)}=\cfrac{|\mathcal{M}_t(i)^+|}{|\mathcal{M}_t(i)^-|}
\end{equation}

We call $\gamma^A_i(t)$ the \emph{convergence coefficient} of local search metaheuristic $A$. We can now study the convergence behaviour of local search 
metaheuristics according to the limit of $\gamma^A_i(t)$ as $t\rightarrow\infty$. If $\lim_{t\rightarrow\infty} \gamma^A_t(t) = 0$, this means 
that algorithm $A$ has converged, since the improving set becomes zero faster than the non-improving set.  

The second interesting measure representes an explicit way of evaluating the balance between exploration and exploitation that a 
specific local search metaheuristic provides. We define

\begin{equation}
    \delta^A_{ia}(t):=\cfrac{\mathbb{P}\{\sigma(a)=1 \, \vert \, S_t=i\}}{\mathbb{P}\{\sigma(a)=0 \, \vert \, S_t=i\}}
\end{equation}

Averaging over all actions and times, we get

\begin{equation}
    \delta^A_i=\sum_{t=0}^{\infty} \sum_{j\in\mathcal{A}(i)} p_{ij}(a) \delta^A_{ja}(t)=\frac{1}{|\mathcal{A}(i)|}\sum_{j\in\mathcal{A}(i)}\sum_{t=0}^\infty \delta_{ja}^A(t)
\end{equation}

We call $\delta^A_i$ the \emph{exploration-exploitation coefficient} of local search metaheuristic $A$ conditioned on state $i\in\mathcal{S}$.
In vector notation, we write $\delta^A$ and call it the \emph{exploration-exploitation coefficient vector}. In general, there can be 
components of $\delta^A$ that converge, and others that diverge. Therefore, one possibility would be to assert the general 
exploration-exploitation behaviour of $A$ based on $\delta_*^A = \max_i \delta_i^A$ and investigate the individual components of 
$\delta^A$ in case a more detailed analysis is needed. 

We can now define for each metaheuristic a measure of its exploration-exploitation balance as follows:

\begin{definition}
    We say that $A$ is a \emph{balanced} local search metaheuristic if $\delta^A_i = C$ for a constant $C>0,C\in\mathbb{R}$. Otherwise, 
    if $\delta^A_i$ diverges to $\infty$, we say that $A$ is \emph{exploration-oriented}. Finally, if 
    $\delta^A_i = 0$ we say that $A$ is \emph{exploitation-oriented}.
\end{definition}

Using our local search MDP, we have transformed the problem of convergence analysis and 
determining the exploration-exploitation balance of any 
metaheuristic into the quantitative problem of calculating $\delta^A_i$ and the limit of $\gamma^A_t(t)$ as $t\rightarrow\infty$ . Let's apply 
this to the hill climbing heuristic already mentioned before.

\begin{example}
    In the case of hill climbing, we have for every $i\in\mathcal{S}$:
    \begin{align}
        \alpha_i^A(t)&=0 \nonumber \\
        \beta_i^A(t)&=1/|\mathcal{M}(i)| \nonumber \\
        \gamma_i^A(t)&=0
    \end{align}
    Since $\lim_{t\rightarrow\infty} \gamma^A_i(t) = 0, \forall i\in\mathcal{S}$, we have that hill climbing always 
    converges to a (possibly local) optimum. Regarding exploration-exploitation behaviour, we have $\delta_{ia}(t)=0, \forall a\in\mathcal{A}(i),\forall i\in\mathcal{S}$.
    We therefore obtain that $\delta_*^A=0$ and hill climbing is exploitation-oriented.
\end{example}

In the next section, we apply the machinery defined to the simulated annealing (SA) metaheuristic.

\section{Simulated Annealing}
\label{sec4}

Simulated annealing (SA) \citep{kirkpatrick_optimization_1983} is a local search metaheuristic inspired by the procedure of cooling 
a molten solid according to a predefined cooling schedule. The main idea is to start with a high temperature $T$, and decrease 
this temperature in subsequent iterations with the goal of reaching the minimum value of an energy function $E$. Transitions from the current candidate 
solution that improve the objective value are accepted with probability 1. 
Otherwise, the solution is accepted with a probability that depends on the magnitude of difference between the current best objective value and 
the value of the new candidate solution. In the initial, high-temperature 
phase, the probability of accepting worse solutions is high (emphasizes exploration), whereas in later stages 
lower temperatures also decrease this probability (emphasizes exploitation). More formally, this probability can be written as:

\begin{equation}
    \mathbb{P}\{ \mathbf{x}_{t+1}=\mathbf{x}' \, \vert \, \mathbf{x}_t = \mathbf{x}\} =
    \begin{cases}
        1 \text{, if } E(\mathbf{x'}) < E(\mathbf{x}) \\
        \exp\left(\cfrac{- \left(E(\mathbf{x}) - E(\mathbf{x'})\right)}{T} \right) \text{, otherwise}
    \end{cases}
\end{equation}

In our case, since we are solving by convention maximization problems, $E=-f$. Usually, a geometric cooling scheme is used: 
$T_{new} = \alpha T_{current}$, where $\alpha\in [0,1)$. Let $T_0$ denote the initial temperature (which is a free 
parameter of the algorithm). Then the temperature at time $t$ is given by $T=\alpha^t T_0$.

We now define a policy $R_{SA}$ to model the behaviour of SA as a stochastic decision agent as follows:

\begin{itemize}
    \item Let $\mathcal{M}_t^+(i)$ and $\mathcal{M}_t^-(i)$ be defined as in Theorem \ref{theorem1}.
    \item Let's consider action $a=i\mapsto j$ and state $i\in\mathcal{S}$. Using Theorem \ref{theorem1} we obtain
    \begin{align}
        \pi^t_{ia}&=\mathbb{P}\{ A_t = a \, \vert \, S_t = i \} \nonumber \\
            &= \frac{|\mathcal{M}_t(i)^+|}{|\mathcal{M}_t(i)^+|+|\mathcal{M}_t(i)^-|}
            +\frac{|\mathcal{M}_t(i)^-|}{|\mathcal{M}_t(i)^+|+|\mathcal{M}_t(i)^-|}\exp\left(-\frac{f(j)-f(i)}{\alpha^t T_0}\right)
    \end{align}
\end{itemize}

Now, the convergence coefficient can be calculated as follows:

\begin{equation}
    \lim_{t\rightarrow\infty}\gamma_i^A(t)=\lim_{t\rightarrow\infty}\cfrac{|\mathcal{M}_t(i)^+|}{|\mathcal{M}_t(i)^-|}\xrightarrow[t\rightarrow\infty]{} 0
\end{equation}

Since when time approaches infinity, assuming a geometric cooling scheme, the algorithm will only accept improving actions and there will 
be many more non-improving than improving actions.

Let's now calculate the exploration-exploitation coefficient:

\begin{equation}
    \delta_i^A=\frac{1}{|\mathcal{A}(i)|}\sum_{j\in\mathcal{A}(i)}\sum_{t=0}^\infty \delta_{ja}^A(t)=
    \frac{1}{|\mathcal{A}(i)|}\sum_{j\in\mathcal{A}(i)}\sum_{t=0}^\infty \exp\left(-\frac{f(j)-f(i)}{\alpha^t T_0}\right)
\end{equation}

which for all $i\in\mathcal{S}$ and $\alpha \in [0,1)$ converges to a positive constant $C>0$, since:

\begin{equation}
    \cfrac{\exp\left(-\frac{f(j)-f(i)}{\alpha^{t+1} T_0}\right)}{\exp\left(-\frac{f(j)-f(i)}{\alpha^t T_0}\right)}=
    \exp\left(-\frac{(1-\alpha)(f(j)-f(i))}{\alpha^{t+1} T_0}\right)\xrightarrow[t\rightarrow\infty]{} 0
\end{equation}

Therefore, simulated annealing is a balanced local search metaheuristic.

\section{Conclusions}
\label{sec5}

In this paper, we have presented a new framework for modeling the convergence and exploration-exploitation 
behaviour of local search metaheuristics based on Markov decision processes. This framework provides a novel and 
intuitive way of characterizing essential properties of local search metaheuristis in a principled way. 
We proved that any local search metaheuristic can be analyzed using the tools provided and we showed 
how to apply the framework to two examples: hill climbing and simulated annealing. In the case of 
hill climbing, we showed that the framework correctly classifies it as exploitation-oriented, whereas 
simulated annealing was classified as balanced, which is consistent with the general consensus in the literature. 
Our study goes beyond the state-of-the-art in metaheuristic analysis, since the tools we develop provide a more 
expressive methodology than methods based on Markov chains and other probabilistic approaches. As a consequence, 
the behaviour of local search metaheuristics can be modeled in a more realistic way and the results obtained are 
consistent with practice.

As future work, we plan to extend our results to other local search metaheuristics like variable neighborhood search (VNS) \citep{mladenovic_variable_1997} and
large neighborhood search (LNS) \citep{pisinger_large_2019}. Additionally, the framework will be extended to be able 
to handle population-based metaheuristics (PBMH) like genetic algorithms (GA) \citep{holland_genetic_1992}, particle swarm 
optimization (PSO) \citep{kennedy_particle_1995} and estimation of distribution algorithms (EDA) \citep{larranaga_estimation_2001}. 

\section*{Declarations}

No funds, grants, or other support was received. The authors have no financial or proprietary interests in any material discussed in this article.


\bibliography{mdp-local-search}


\begin{thebibliography}{30}
\ifx \bisbn   \undefined \def \bisbn  #1{ISBN #1}\fi
\ifx \binits  \undefined \def \binits#1{#1}\fi
\ifx \bauthor  \undefined \def \bauthor#1{#1}\fi
\ifx \batitle  \undefined \def \batitle#1{#1}\fi
\ifx \bjtitle  \undefined \def \bjtitle#1{#1}\fi
\ifx \bvolume  \undefined \def \bvolume#1{\textbf{#1}}\fi
\ifx \byear  \undefined \def \byear#1{#1}\fi
\ifx \bissue  \undefined \def \bissue#1{#1}\fi
\ifx \bfpage  \undefined \def \bfpage#1{#1}\fi
\ifx \blpage  \undefined \def \blpage #1{#1}\fi
\ifx \burl  \undefined \def \burl#1{\textsf{#1}}\fi
\ifx \doiurl  \undefined \def \doiurl#1{\url{https://doi.org/#1}}\fi
\ifx \betal  \undefined \def \betal{\textit{et al.}}\fi
\ifx \binstitute  \undefined \def \binstitute#1{#1}\fi
\ifx \binstitutionaled  \undefined \def \binstitutionaled#1{#1}\fi
\ifx \bctitle  \undefined \def \bctitle#1{#1}\fi
\ifx \beditor  \undefined \def \beditor#1{#1}\fi
\ifx \bpublisher  \undefined \def \bpublisher#1{#1}\fi
\ifx \bbtitle  \undefined \def \bbtitle#1{#1}\fi
\ifx \bedition  \undefined \def \bedition#1{#1}\fi
\ifx \bseriesno  \undefined \def \bseriesno#1{#1}\fi
\ifx \blocation  \undefined \def \blocation#1{#1}\fi
\ifx \bsertitle  \undefined \def \bsertitle#1{#1}\fi
\ifx \bsnm \undefined \def \bsnm#1{#1}\fi
\ifx \bsuffix \undefined \def \bsuffix#1{#1}\fi
\ifx \bparticle \undefined \def \bparticle#1{#1}\fi
\ifx \barticle \undefined \def \barticle#1{#1}\fi
\bibcommenthead
\ifx \bconfdate \undefined \def \bconfdate #1{#1}\fi
\ifx \botherref \undefined \def \botherref #1{#1}\fi
\ifx \url \undefined \def \url#1{\textsf{#1}}\fi
\ifx \bchapter \undefined \def \bchapter#1{#1}\fi
\ifx \bbook \undefined \def \bbook#1{#1}\fi
\ifx \bcomment \undefined \def \bcomment#1{#1}\fi
\ifx \oauthor \undefined \def \oauthor#1{#1}\fi
\ifx \citeauthoryear \undefined \def \citeauthoryear#1{#1}\fi
\ifx \endbibitem  \undefined \def \endbibitem {}\fi
\ifx \bconflocation  \undefined \def \bconflocation#1{#1}\fi
\ifx \arxivurl  \undefined \def \arxivurl#1{\textsf{#1}}\fi
\csname PreBibitemsHook\endcsname

\bibitem[\protect\citeauthoryear{Cuevas
  et~al.}{2021}]{cuevas_experimental_2021}
\begin{bchapter}
\bauthor{\bsnm{Cuevas}, \binits{E.}},
\bauthor{\bsnm{Diaz}, \binits{P.}},
\bauthor{\bsnm{Camarena}, \binits{O.}}:
\bctitle{Experimental {Analysis} {Between} {Exploration} and {Exploitation}}.
In: \beditor{\bsnm{Cuevas}, \binits{E.}},
\beditor{\bsnm{Diaz}, \binits{P.}},
\beditor{\bsnm{Camarena}, \binits{O.}} (eds.)
\bbtitle{Metaheuristic {Computation}: {A} {Performance} {Perspective}}.
\bsertitle{Intelligent {Systems} {Reference} {Library}},
pp. \bfpage{249}--\blpage{269}.
\bpublisher{Springer},
\blocation{Cham}
(\byear{2021}).
\doiurl{10.1007/978-3-030-58100-8_10} .
\burl{https://doi.org/10.1007/978-3-030-58100-8_10}
Accessed 2022-06-29
\end{bchapter}
\endbibitem

\bibitem[\protect\citeauthoryear{Ceschia et~al.}{2020}]{CeGS20}
\begin{barticle}
\bauthor{\bsnm{Ceschia}, \binits{S.}},
\bauthor{\bsnm{Guido}, \binits{R.}},
\bauthor{\bsnm{Schaerf}, \binits{A.}}:
\batitle{Solving the static inrc-ii nurse rostering problem by simulated
  annealing based on large neighborhoods}.
\bjtitle{Annals of Operations Research}
\bvolume{288},
\bfpage{95}--\blpage{113}
(\byear{2020})
\doiurl{10.1007/s10479-020-03527-6}
\end{barticle}
\endbibitem

\bibitem[\protect\citeauthoryear{Chen and He}{2020}]{chen_exploitation_2020}
\begin{botherref}
\oauthor{\bsnm{Chen}, \binits{Y.}},
\oauthor{\bsnm{He}, \binits{J.}}:
Exploitation and {Exploration} {Analysis} of {Elitist} {Evolutionary}
  {Algorithms}: {A} {Case} {Study}.
arXiv.
arXiv:2001.10932 [cs]
(2020).
\doiurl{10.48550/arXiv.2001.10932} .
\url{http://arxiv.org/abs/2001.10932}
Accessed 2022-07-12
\end{botherref}
\endbibitem

\bibitem[\protect\citeauthoryear{Chauhdry}{2023}]{chauhdry_framework_2023}
\begin{barticle}
\bauthor{\bsnm{Chauhdry}, \binits{M.H.M.}}:
\batitle{A framework using nested partitions algorithm for convergence analysis
  of population distribution-based methods}.
\bjtitle{EURO Journal on Computational Optimization}
\bvolume{11},
\bfpage{100067}
(\byear{2023})
\doiurl{10.1016/j.ejco.2023.100067} .
Accessed 2024-07-01
\end{barticle}
\endbibitem

\bibitem[\protect\citeauthoryear{Cao and Wu}{1997}]{cao_convergence_1997}
\begin{bchapter}
\bauthor{\bsnm{Cao}, \binits{Y.J.}},
\bauthor{\bsnm{Wu}, \binits{Q.H.}}:
\bctitle{Convergence analysis of adaptive genetic algorithms}.
In: \bbtitle{Second {International} {Conference} {On} {Genetic} {Algorithms}
  {In} {Engineering} {Systems}: {Innovations} {And} {Applications}},
pp. \bfpage{85}--\blpage{89}
(\byear{1997}).
\doiurl{10.1049/cp:19971160} .
\bcomment{ISSN: 0537-9989}.
\burl{https://ieeexplore.ieee.org/document/680987}
Accessed 2024-07-29
\end{bchapter}
\endbibitem

\bibitem[\protect\citeauthoryear{Chen et~al.}{2009}]{chen_optimal_2009}
\begin{barticle}
\bauthor{\bsnm{Chen}, \binits{J.}},
\bauthor{\bsnm{Xin}, \binits{B.}},
\bauthor{\bsnm{Peng}, \binits{Z.}},
\bauthor{\bsnm{Dou}, \binits{L.}},
\bauthor{\bsnm{Zhang}, \binits{J.}}:
\batitle{Optimal {Contraction} {Theorem} for {Exploration}–{Exploitation}
  {Tradeoff} in {Search} and {Optimization}}.
\bjtitle{IEEE Transactions on Systems, Man, and Cybernetics - Part A: Systems
  and Humans}
\bvolume{39}(\bissue{3}),
\bfpage{680}--\blpage{691}
(\byear{2009})
\doiurl{10.1109/TSMCA.2009.2012436} .
\bcomment{Conference Name: IEEE Transactions on Systems, Man, and Cybernetics -
  Part A: Systems and Humans}
\end{barticle}
\endbibitem

\bibitem[\protect\citeauthoryear{Doering
  et~al.}{2019}]{doering_metaheuristics_2019}
\begin{barticle}
\bauthor{\bsnm{Doering}, \binits{J.}},
\bauthor{\bsnm{Kizys}, \binits{R.}},
\bauthor{\bsnm{Juan}, \binits{A.A.}},
\bauthor{\bsnm{Fitó}, \binits{n.}},
\bauthor{\bsnm{Polat}, \binits{O.}}:
\batitle{Metaheuristics for rich portfolio optimisation and risk management:
  {Current} state and future trends}.
\bjtitle{Operations Research Perspectives}
\bvolume{6},
\bfpage{100121}
(\byear{2019})
\doiurl{10.1016/j.orp.2019.100121} .
Accessed 2024-07-20
\end{barticle}
\endbibitem

\bibitem[\protect\citeauthoryear{Eiben et~al.}{1991}]{eiben_global_1991}
\begin{bchapter}
\bauthor{\bsnm{Eiben}, \binits{A.E.}},
\bauthor{\bsnm{Aarts}, \binits{E.H.L.}},
\bauthor{\bsnm{Van~Hee}, \binits{K.M.}}:
\bctitle{Global convergence of genetic algorithms: {A} markov chain analysis}.
In: \beditor{\bsnm{Schwefel}, \binits{H.-P.}},
\beditor{\bsnm{Männer}, \binits{R.}} (eds.)
\bbtitle{Parallel {Problem} {Solving} from {Nature}},
pp. \bfpage{3}--\blpage{12}.
\bpublisher{Springer},
\blocation{Berlin, Heidelberg}
(\byear{1991}).
\doiurl{10.1007/BFb0029725}
\end{bchapter}
\endbibitem

\bibitem[\protect\citeauthoryear{Eberbach}{2005}]{eberbach_toward_2005}
\begin{barticle}
\bauthor{\bsnm{Eberbach}, \binits{E.}}:
\batitle{Toward a theory of evolutionary computation}.
\bjtitle{Biosystems}
\bvolume{82}(\bissue{1}),
\bfpage{1}--\blpage{19}
(\byear{2005})
\doiurl{10.1016/j.biosystems.2005.05.006} .
Accessed 2024-07-29
\end{barticle}
\endbibitem

\bibitem[\protect\citeauthoryear{Ecoretti et~al.}{2022}]{EcCS22}
\begin{bchapter}
\bauthor{\bsnm{Ecoretti}, \binits{A.}},
\bauthor{\bsnm{Ceschia}, \binits{S.}},
\bauthor{\bsnm{Schaerf}, \binits{A.}}:
\bctitle{Local search for integrated predictive maintenance and scheduling in
  flow-shop}.
In: \bbtitle{14th Metaheuristics International Conference}
(\byear{2022})
\end{bchapter}
\endbibitem

\bibitem[\protect\citeauthoryear{Fernandez
  et~al.}{2018}]{fernandez_metaheuristics_2018}
\begin{barticle}
\bauthor{\bsnm{Fernandez}, \binits{S.A.}},
\bauthor{\bsnm{Juan}, \binits{A.A.}},
\bauthor{\bsnm{Armas~Adrián}, \binits{J.}},
\bauthor{\bsnm{Silva}, \binits{D.G.e.}},
\bauthor{\bsnm{Terrén}, \binits{D.R.}}:
\batitle{Metaheuristics in {Telecommunication} {Systems}: {Network} {Design},
  {Routing}, and {Allocation} {Problems}}.
\bjtitle{IEEE Systems Journal}
\bvolume{12}(\bissue{4}),
\bfpage{3948}--\blpage{3957}
(\byear{2018})
\doiurl{10.1109/JSYST.2017.2788053} .
\bcomment{Conference Name: IEEE Systems Journal}.
Accessed 2024-07-20
\end{barticle}
\endbibitem

\bibitem[\protect\citeauthoryear{Glover and Laguna}{1998}]{glover_tabu_1998}
\begin{bbook}
\bauthor{\bsnm{Glover}, \binits{F.}},
\bauthor{\bsnm{Laguna}, \binits{M.}}:
\bbtitle{Tabu Search}.
\bpublisher{Springer}, \blocation{???}
(\byear{1998})
\end{bbook}
\endbibitem

\bibitem[\protect\citeauthoryear{Holland}{1975}]{holland_adaptation_1975}
\begin{botherref}
\oauthor{\bsnm{Holland}, \binits{J.H.}}:
Adaptation in {Natural} and {Artificial} {Systems}
(1975).
\url{https://mitpress.mit.edu/9780262581110/adaptation-in-natural-and-artificial-systems/}
Accessed 2024-07-29
\end{botherref}
\endbibitem

\bibitem[\protect\citeauthoryear{Holland}{1992}]{holland_genetic_1992}
\begin{barticle}
\bauthor{\bsnm{Holland}, \binits{J.H.}}:
\batitle{Genetic {Algorithms}}.
\bjtitle{Scientific American}
\bvolume{267}(\bissue{1}),
\bfpage{66}--\blpage{73}
(\byear{1992}).
\bcomment{Publisher: Scientific American, a division of Nature America, Inc.}
\end{barticle}
\endbibitem

\bibitem[\protect\citeauthoryear{Kallenberg}{2022}]{kallenberg_lecture_2022}
\begin{bbook}
\bauthor{\bsnm{Kallenberg}, \binits{L.C.M.}}:
\bbtitle{Lecture {Notes} on {Markov} {Decision} {Processes}}.
\bpublisher{University of Leiden}, \blocation{???}
(\byear{2022})
\end{bbook}
\endbibitem

\bibitem[\protect\citeauthoryear{Kennedy and
  Eberhart}{1995}]{kennedy_particle_1995}
\begin{bchapter}
\bauthor{\bsnm{Kennedy}, \binits{J.}},
\bauthor{\bsnm{Eberhart}, \binits{R.}}:
\bctitle{Particle swarm optimization}.
In: \bbtitle{Proceedings of {ICNN}'95-international Conference on Neural
  Networks},
vol. \bseriesno{4},
pp. \bfpage{1942}--\blpage{1948}.
\bpublisher{ieee}, \blocation{???}
(\byear{1995})
\end{bchapter}
\endbibitem

\bibitem[\protect\citeauthoryear{Kirkpatrick
  et~al.}{1983}]{kirkpatrick_optimization_1983}
\begin{barticle}
\bauthor{\bsnm{Kirkpatrick}, \binits{S.}},
\bauthor{\bsnm{Gelatt~Jr}, \binits{C.D.}},
\bauthor{\bsnm{Vecchi}, \binits{M.P.}}:
\batitle{Optimization by simulated annealing}.
\bjtitle{science}
\bvolume{220}(\bissue{4598}),
\bfpage{671}--\blpage{680}
(\byear{1983}).
\bcomment{Publisher: American association for the advancement of science}
\end{barticle}
\endbibitem

\bibitem[\protect\citeauthoryear{Kaur and Saini}{2021}]{kaur_review_2021}
\begin{bchapter}
\bauthor{\bsnm{Kaur}, \binits{M.}},
\bauthor{\bsnm{Saini}, \binits{S.}}:
\bctitle{A {Review} of {Metaheuristic} {Techniques} for {Solving} {University}
  {Course} {Timetabling} {Problem}}.
In: \beditor{\bsnm{Goar}, \binits{V.}},
\beditor{\bsnm{Kuri}, \binits{M.}},
\beditor{\bsnm{Kumar}, \binits{R.}},
\beditor{\bsnm{Senjyu}, \binits{T.}} (eds.)
\bbtitle{Advances in {Information} {Communication} {Technology} and
  {Computing}}.
\bsertitle{Lecture {Notes} in {Networks} and {Systems}},
pp. \bfpage{19}--\blpage{25}.
\bpublisher{Springer},
\blocation{Singapore}
(\byear{2021}).
\doiurl{10.1007/978-981-15-5421-6_3}
\end{bchapter}
\endbibitem

\bibitem[\protect\citeauthoryear{Larrañaga and
  Lozano}{2001}]{larranaga_estimation_2001}
\begin{bbook}
\bauthor{\bsnm{Larrañaga}, \binits{P.}},
\bauthor{\bsnm{Lozano}, \binits{J.A.}}:
\bbtitle{Estimation of Distribution Algorithms: {A} New Tool for Evolutionary
  Computation}
vol. \bseriesno{2}.
\bpublisher{Springer}, \blocation{???}
(\byear{2001})
\end{bbook}
\endbibitem

\bibitem[\protect\citeauthoryear{Lin et~al.}{2009}]{lin_applying_2009}
\begin{barticle}
\bauthor{\bsnm{Lin}, \binits{S.-W.}},
\bauthor{\bsnm{Lee}, \binits{Z.-J.}},
\bauthor{\bsnm{Ying}, \binits{K.-C.}},
\bauthor{\bsnm{Lee}, \binits{C.-Y.}}:
\batitle{Applying hybrid meta-heuristics for capacitated vehicle routing
  problem}.
\bjtitle{Expert Systems with Applications}
\bvolume{36}(\bissue{2}),
\bfpage{1505}--\blpage{1512}
(\byear{2009}).
\bcomment{Publisher: Elsevier}
\end{barticle}
\endbibitem

\bibitem[\protect\citeauthoryear{Liu et~al.}{2011}]{liu_unified_2011}
\begin{barticle}
\bauthor{\bsnm{Liu}, \binits{B.}},
\bauthor{\bsnm{Wang}, \binits{L.}},
\bauthor{\bsnm{Liu}, \binits{Y.}},
\bauthor{\bsnm{Wang}, \binits{S.}}:
\batitle{A unified framework for population-based metaheuristics}.
\bjtitle{Annals of Operations Research}
\bvolume{186}(\bissue{1}),
\bfpage{231}--\blpage{262}
(\byear{2011})
\doiurl{10.1007/s10479-011-0894-3} .
Accessed 2024-06-26
\end{barticle}
\endbibitem

\bibitem[\protect\citeauthoryear{Mladenović and
  Hansen}{1997}]{mladenovic_variable_1997}
\begin{barticle}
\bauthor{\bsnm{Mladenović}, \binits{N.}},
\bauthor{\bsnm{Hansen}, \binits{P.}}:
\batitle{Variable neighborhood search}.
\bjtitle{Computers \& operations research}
\bvolume{24}(\bissue{11}),
\bfpage{1097}--\blpage{1100}
(\byear{1997}).
\bcomment{Publisher: Elsevier}
\end{barticle}
\endbibitem

\bibitem[\protect\citeauthoryear{Pillay}{2014}]{pillay_survey_2014}
\begin{barticle}
\bauthor{\bsnm{Pillay}, \binits{N.}}:
\batitle{A survey of school timetabling research}.
\bjtitle{Annals of Operations Research}
\bvolume{218}(\bissue{1}),
\bfpage{261}--\blpage{293}
(\byear{2014})
\doiurl{10.1007/s10479-013-1321-8} .
Accessed 2022-03-03
\end{barticle}
\endbibitem

\bibitem[\protect\citeauthoryear{Pisinger and
  Ropke}{2019}]{pisinger_large_2019}
\begin{bchapter}
\bauthor{\bsnm{Pisinger}, \binits{D.}},
\bauthor{\bsnm{Ropke}, \binits{S.}}:
\bctitle{Large {Neighborhood} {Search}}.
In: \beditor{\bsnm{Gendreau}, \binits{M.}},
\beditor{\bsnm{Potvin}, \binits{J.-Y.}} (eds.)
\bbtitle{Handbook of {Metaheuristics}},
pp. \bfpage{99}--\blpage{127}.
\bpublisher{Springer},
\blocation{Cham}
(\byear{2019}).
\doiurl{10.1007/978-3-319-91086-4_4} .
\burl{https://doi.org/10.1007/978-3-319-91086-4_4}
Accessed 2024-07-29
\end{bchapter}
\endbibitem

\bibitem[\protect\citeauthoryear{Silberholz and
  Golden}{2010}]{silberholz_comparison_2010}
\begin{bchapter}
\bauthor{\bsnm{Silberholz}, \binits{J.}},
\bauthor{\bsnm{Golden}, \binits{B.}}:
\bctitle{Comparison of {Metaheuristics}}.
In: \beditor{\bsnm{Gendreau}, \binits{M.}},
\beditor{\bsnm{Potvin}, \binits{J.-Y.}} (eds.)
\bbtitle{Handbook of {Metaheuristics}},
pp. \bfpage{625}--\blpage{640}.
\bpublisher{Springer},
\blocation{Boston, MA}
(\byear{2010}).
\doiurl{10.1007/978-1-4419-1665-5_21} .
\burl{https://doi.org/10.1007/978-1-4419-1665-5_21}
Accessed 2024-07-20
\end{bchapter}
\endbibitem

\bibitem[\protect\citeauthoryear{Sörensen and
  Glover}{2013}]{sorensen_metaheuristics_2013}
\begin{barticle}
\bauthor{\bsnm{Sörensen}, \binits{K.}},
\bauthor{\bsnm{Glover}, \binits{F.}}:
\batitle{Metaheuristics}.
\bjtitle{Encyclopedia of operations research and management science}
\bvolume{62},
\bfpage{960}--\blpage{970}
(\byear{2013}).
\bcomment{Publisher: Springer Boston, MA, USA}
\end{barticle}
\endbibitem

\bibitem[\protect\citeauthoryear{Shaw}{1998}]{shaw_using_1998}
\begin{bchapter}
\bauthor{\bsnm{Shaw}, \binits{P.}}:
\bctitle{Using constraint programming and local search methods to solve vehicle
  routing problems}.
In: \bbtitle{International Conference on Principles and Practice of Constraint
  Programming},
pp. \bfpage{417}--\blpage{431}.
\bpublisher{Springer}, \blocation{???}
(\byear{1998})
\end{bchapter}
\endbibitem

\bibitem[\protect\citeauthoryear{Suzuki}{1995}]{suzuki1995markov}
\begin{barticle}
\bauthor{\bsnm{Suzuki}, \binits{J.}}:
\batitle{A markov chain analysis on simple genetic algorithms}.
\bjtitle{IEEE Transactions on Systems, Man, and Cybernetics}
\bvolume{25}(\bissue{4}),
\bfpage{655}--\blpage{659}
(\byear{1995})
\end{barticle}
\endbibitem

\bibitem[\protect\citeauthoryear{Suzuki}{1998}]{suzuki1998further}
\begin{barticle}
\bauthor{\bsnm{Suzuki}, \binits{J.}}:
\batitle{A further result on the markov chain model of genetic algorithms and
  its application to a simulated annealing-like strategy}.
\bjtitle{IEEE Transactions on Systems, Man, and Cybernetics, Part B
  (Cybernetics)}
\bvolume{28}(\bissue{1}),
\bfpage{95}--\blpage{102}
(\byear{1998})
\end{barticle}
\endbibitem

\bibitem[\protect\citeauthoryear{Xu and
  Zhang}{2014}]{xu_exploration-exploitation_2014}
\begin{bchapter}
\bauthor{\bsnm{Xu}, \binits{J.}},
\bauthor{\bsnm{Zhang}, \binits{J.}}:
\bctitle{Exploration-exploitation tradeoffs in metaheuristics: {Survey} and
  analysis}.
In: \bbtitle{Proceedings of the 33rd {Chinese} {Control} {Conference}},
pp. \bfpage{8633}--\blpage{8638}
(\byear{2014}).
\doiurl{10.1109/ChiCC.2014.6896450} .
\bcomment{ISSN: 1934-1768}
\end{bchapter}
\endbibitem

\end{thebibliography}

\end{document}